\documentclass[twoside]{article}
\usepackage[utf8]{inputenc}
\usepackage{amsmath} 
\usepackage{amssymb}
\usepackage{amsthm}
\usepackage{graphicx}
\usepackage{algorithm, algorithmic}

\newtheorem{theorem}{Theorem}
\newtheorem{lemma}{Lemma}
\newtheorem{definition}{Definition}
\newtheorem{corollary}{Corollary}
\DeclareMathOperator{\Tr}{Tr}

\usepackage[accepted]{aistats2019}

\begin{document}

\twocolumn[

\aistatstitle{An Online Algorithm for Smoothed Regression and LQR Control}
\aistatsauthor{ Gautam Goel \And Adam Wierman }

\aistatsaddress{ Caltech \And  Caltech } 
]

\begin{abstract}
We consider Online Convex Optimization (OCO) in the setting where the costs are $m$-strongly convex and the online learner pays a switching cost for changing decisions between rounds. We show that the recently proposed Online Balanced Descent (OBD) algorithm  is constant competitive in this setting, with competitive ratio $3 + O(1/m)$, irrespective of the ambient dimension.  Additionally, we show that when the sequence of cost functions is $\epsilon$-smooth, OBD has near-optimal dynamic regret and maintains strong per-round accuracy. We demonstrate the generality of our approach by showing that the OBD framework can be used to construct competitive algorithms for a variety of online problems across learning and control, including online variants of ridge regression, logistic regression, maximum likelihood estimation, and LQR control.

\end{abstract}

\section{Introduction}

In this paper we study the problem of \textit{smoothed online convex optimization} (SOCO), a variant of OCO where the online learner incurs a \textit{switching cost} for changing its actions between rounds. More concretely, the online learner plays a series of rounds $t = 1 \ldots T$. In each round, the learner receives a convex loss function $f_t$, picks a point $x_t$ from a convex action space $\mathcal{\chi} \subset \mathbb{R}^d$, and pays a \textit{hitting cost} $f_t(x_t)$ as well as a \textit{switching cost} $c(x_t, x_{t-1})$ which penalizes the learner for changing its action between rounds. 

This problem was first introduced in the context of the dynamic management of service capacity in data centers \cite{5934885}, where the switching costs represent the performance and wear-and-tear costs associated with changing server configurations.  Since then, SOCO has attracted considerable interest, both theoretical and applied, due to its use in dozens of applications across learning, distributed systems, networking, and control, such as speech animation \cite{kim2015decision}, video streaming \cite{6195799}, management of electric vehicle charging \cite{kim2014real}, geographical load balancing \cite{6322266}, and multi-timescale control \cite{goel2017thinking}. See \cite{pmlr-v75-chen18b} for an extensive list of applications.  

Unfortunately, despite a large and growing literature, all existing results identifying competitive algorithms for SOCO either (i) place strong restrictions on the action space, (ii), place strong restrictions on the class of loss functions, or (iii) require algorithms to make use of predictions of future cost functions. For example, a series of papers \cite{5934885}, \cite{bansal20152} developed competitive algorithms for one-dimensional action spaces.  Until earlier this year there were no known algorithms that were competitive for SOCO beyond one dimension without requiring the use of predictions.  Finally, \cite{pmlr-v75-chen18b} presented the first algorithm that is constant-competitive beyond one dimension, but the algorithm was shown to be constant competitive only in the case of polyhedral cost functions, a restrictive class that does not include most loss functions used in machine learning.  Beyond this result, the most general positive results all assume predictions of future cost functions are available, e.g. \cite{6322266}, \cite{chen2015online}, \cite{chen2016using}, \cite{li2018using}.  

The existing work on SOCO highlights a crucial open question: \textit{Does there exist a competitive algorithm for high-dimensional SOCO problems with cost functions that capture standard losses for online learning problems, e.g., logistic loss or least-squares loss?}  



In this paper we answer this question by proving that the recently introduced Online Balanced Descent (OBD) algorithm is constant-competitive for SOCO with strongly convex costs.  Additionally, highlighting the importance of the class of strongly convex costs, we show that the OBD framework can be used to construct the first competitive algorithms for problems as diverse as online ridge regression, online logistic regression, and LQR control, which was not possible with previous approaches.

\textbf{Contributions of this paper.}  This paper makes three main contributions to the literature on SOCO.

First, in Section \ref{cr-sec} we show that OBD is constant competitive for SOCO when the costs are strongly convex (Theorem \ref{cr-thm}). This establishes OBD as the first constant competitive algorithm for strongly convex costs beyond one-dimensional action spaces. The key to our proof is a novel potential function argument that exploits essential properties of $\ell_2$ geometry. In particular, controlling the change in potential rests upon comparing the side lengths of certain triangles (see Figure \ref{triangle-fig}), which can be done via the Law of Cosines. 

Secondly, in Section \ref{bwc-sec} we adopt a beyond-worst-case perspective and show that when the sequence of cost functions
does not vary too much between rounds, OBD guarantees per-step accuracy (Theorem \ref{smooth-thm}), meaning that the point OBD picks is always near the current minimizer. This is attractive from a statistical perspective, since it lets us bound the loss in accuracy we incur in each round due to the effects of the switching costs. We also show that OBD has near-optimal dynamic regret, almost matching a lower bound of \cite{li2018using}.  Specifically, in Theorem \ref{regret-thm} we show that OBD has dynamic regret $O((\epsilon + \epsilon^2) T)$, where $\epsilon$ is a parameter that controls how much the cost sequence varies across rounds.

Finally, in Section \ref{applications-sec} we show novel applications of OBD to  problems arising in statistics, learning, and control, including ridge regression, logistic regression, and maximum likelihood estimation. A highlight of this section is a reduction of LQR control to SOCO, giving the first competitive algorithm for LQR control (results for LQR control typically make strong distributional assumptions). We emphasize that none of these applications could be handled by previous work on SOCO, which highlights the importance of deriving a competitive bound for OBD in the strongly convex setting.

\textbf{Related work.}
There is a vast literature on OCO; for a recent survey see \cite{hazan2016introduction}. OCO with switching costs was first studied in the scalar setting in \cite{5934885}, which used SOCO to model dynamic right-sizing in data centers and gave a 3-competitive algorithm. In subsequent work, \cite{bansal20152} improved the competitive ratio to 2, also in the scalar setting.  The first constant-competitive algorithm beyond one dimension was given in \cite{pmlr-v75-chen18b}, which introduced the OBD framework and showed that it was competitive for SOCO with polyhedral costs. The results in this paper highlight that OBD is also constant-competitive for strongly convex cost functions, a class that is particularly important for learning and control applications, and is wholly disjoint from the class of polyhedral cost functions when the minimizer of the cost function is zero.  

A special case of SOCO is the Convex Body Chasing problem, first introduced in \cite{friedman1993convex}.  The connection between Convex Body Chasing and SOCO was observed in \cite{antoniadis2016chasing}. A recent series of papers \cite{bansa2018nested}, \cite{argue2018nearly} identified competitive algorithms in the setting where the bodies are nested.

Before this paper, the only class of SOCO problems for which positive results for strongly convex cost functions existed is when the learner had access to accurate predictions of future cost functions.  The study of SOCO with predictions began with \cite{6322266} and then continued with a stream of work in the following years, e.g., \cite{chen2015online,chen2016using}.  The most relevant to this work is \cite{li2018using}, which shows a lower bound on the dynamic regret of SOCO with strongly convex cost functions;  in Section \ref{bwc-sec} we show that OBD can almost match this lower bound.  

In Section \ref{applications-sec} we apply OBD to diverse problems like maximum likelihood estimation and LQR control. These problems have been widely studied; we refer the reader to \cite{boyd2004convex} and 
\cite{astrom2010feedback} for a survey.

Finally, we note that SOCO can be viewed as a continuous version of the classic Metrical Task Systems (MTS) problem, one of the most widely studied problems in the online algorithms community, e.g \cite{borodin1992optimal}, \cite{bartal1997polylog}, \cite{blum2000line}. A special case of the MTS is the celebrated $k$-server problem, first proposed in \cite{manasse1990competitive}, which has received significant attention in recent years, for example in \cite{bubeck2018}.

\section{Smoothed Online Convex Optimization }

An instance of SOCO consists of a convex action set $\mathcal{\chi} \subset \mathbb{R}^d$, an initial point $x_0 \in \mathcal{\chi}$, a sequence of non-negative convex costs $f_1 \ldots f_T : \mathbb{R}^d \rightarrow \mathbb{R}^+$, and a non-negative function $c : \mathbb{R}^d \times \mathbb{R}^d \rightarrow \mathbb{R}^+$. In each round $t$, the online learner observes the cost function $f_t$, picks a point $x_t$, and pays the sum of the \textit{hitting cost} $f_t(x_t)$ and the \textit{movement} or \textit{switching cost} $c(x_t, x_{t-1})$. The switching cost acts as a regularizer, penalizing the online learner for changing its decisions between rounds. The goal of the online learner is to minimize its aggregate cost so as to approximate the offline optimal cost:
\begin{eqnarray*}
\min_{x_1 \ldots x_T \in \mathcal{\chi}} \sum_{t=1}^T f_t(x_t) + c(x_t, x_{t-1}) 
\end{eqnarray*}
More generally, $x_t$ could be matrix-valued and $f_t, c$ could be functions on matrices. Note that we make no restrictions on the sequence of cost functions $f_1 \ldots f_T$ other than strong convexity; they could be adversarial, or even adaptively chosen to hurt the online learner.

We emphasize that SOCO differs from OCO in two important ways. Firstly, unlike in OCO, the costs incurred in each round of SOCO depend on the previous choice, coupling the online learner's decisions across rounds. Secondly, the online learner can observe the cost function $f_t$ before picking $x_t$. This is a standard assumption in the SOCO literature, e.g. in \cite{bansal20152}, \cite{5934885}, \cite{pmlr-v75-chen18b} and isolates the complexity of SOCO onto the coupling across timesteps due to the switching costs instead of the uncertainty in the costs.



In this paper, we measure the performance of OBD in terms of its \textit{dynamic regret} and \textit{competitive ratio}. The dynamic regret is defined as  
\begin{eqnarray*}
 \sum_{t = 1}^T f_t(x_t) + c(x_t, x_{t-1})  
-  \left[\sum_{t = 1}^T f_t(x_t^*) + c(x_t^*, x_{t-1}^*) \right]
\end{eqnarray*} 
Here $x_1 \ldots x_T$ are the points picked by the online learner and $x_1^* \ldots x_T^*$ are the offline optimal points. We note that this is a more natural performance metric for SOCO than static regret, since the main motivation for SOCO is to understand the effects of switching costs on online learning. In contrast, in the static regret setting the comparator never moves and hence incurs no switching cost, making it a less ideal performance metric for SOCO. 

Instead of using an additive metric the competitive ratio uses a multiplicative metric:
$$ \frac{\sum_{t = 1}^T f_t(x_t) + c(x_t, x_{t-1})}{ \sum_{t = 1}^T f_t(x_t^*) + c(x_t^*, x_{t-1}^*)}  $$

We note that \cite{andrew2013tale} showed that, in general, no online algorithm can have both sublinear static regret and constant competitive ratio.

Much attention has been focused on the setting where the switching cost is a norm: $c(x_t, x_{t-1}) = \|x_t - x_{t-1}\|$, e.g. \cite{5934885}, \cite{bansal20152}. Note that in the one-dimensional setting, all $\ell_p$ norms are identical, making the choice of norm somewhat vacuous. The first algorithm to work beyond the one dimensional setting was proposed in \cite{pmlr-v75-chen18b}, which considered a setting where the switching cost is given by the Euclidean distance and the loss functions are \textit{polyhedral}, meaning that they at grow at least linearly as one moves away from the minimizer.

We instead focus on the setting where the cost functions $f_1 \ldots f_T$ are $m$-strongly convex with respect to the Euclidean norm and the switching cost is quadratic: $$c(x_t, x_{t-1}) = \frac{1}{2}\|x_t - x_{t-1}\|_2^2$$
In Section \ref{applications-sec}, we show that OBD can be used with many important loss functions, such as the least-squares loss and the $\ell_2$ regularized logistic loss, none of which could be handled by previous work.

We assume that the domain $\mathcal{\chi}$ is all of $\mathcal{R}^d$. Note that this presents no real restriction, since we can always define $f_t(x) = \infty$ for all $x \not \in \mathcal{\chi}$.  The objective becomes
\begin{eqnarray}
 \min_{x_1 \ldots x_T \in \mathbb{R}^d} \sum_{t=1}^T f_t(x_t) + \frac{1}{2}\|x_t - x_{t-1}\|_2^2 \label{oco-prob}
\end{eqnarray}

\textbf{Notation.} We use $\| \cdot \|$ to denote the $\ell_2$ norm. We often use $H_t$ and $M_t$ to denote the hitting cost $f_t(x_t)$ and the movement cost $\frac{1}{2} \|x_t - x_{t-1}\|^2$, respectively. The offline costs $H_t^*$ and $M_t^*$ are defined analogously. We let $ALG$ denote the total cost incurred by OBD across all rounds and define $OPT$ to be the analogous offline cost. We let $v_t$ denote the minimizer of the cost function $f_t$.


\section{A Competitive Algorithm} \label{cr-sec}

\begin{algorithm}[t]
	\begin{algorithmic}[1]
			\FOR{$t=1, \ldots, T$}
            \STATE Receive $f_t$. Let $v_t$ be the minimizer of $f_t$.     
			\STATE Let $x(l) = \Pi_{K_t^l} (x_{t-1})$. Initialize $l = f_t(v_t)$. Here $K_t^l$ is the $l$-sublevel set of $f_t$, i.e., $K_t^l = \{ x \mid f_t(x) \le l\}$. 
            
            \STATE Increase $l$. Stop either when $x(\ell) = v_t$ or $\frac{1}{2}\|x(l) - x_{t-1}\|_2^2 = \beta l$. 
			\label{projection-step-cr}
            \STATE Set $x_t = x(l)$.
			\ENDFOR
	\end{algorithmic}
\caption{Online Balanced Descent (OBD)}
\label{alg: obd-alg}
\end{algorithm}


Our main technical result shows that a recently proposed algorithm, Online Balanced Descent (OBD), is constant competitive for SOCO problems with strongly convex cost functions. 

OBD was introduced in \cite{pmlr-v75-chen18b}, where it was analyzed for the class of polyhedral costs.  The detailed workings of OBD are summarized in Algorithm 1. The key insight of OBD is to exploit the full geometry of the level sets of the current cost function $f_t$ when choosing the point $x_t$ in such a way as to take switching costs into account. 

OBD works by iteratively projecting the previously chosen point onto a level set of the current cost function. The level set $K_t$ picked by OBD is the level set such that the switching cost incurred while traveling from $x_{t-1}$ to $K_t$ is equal to $\beta f_t(x_t)$, where $x_t$ is the projection onto $K_t$ and $\beta$ is the \textit{balance parameter} which can be tuned to get different performance guarantees. We note that OBD can be efficiently implemented via a binary search over the level sets \cite{pmlr-v75-chen18b}.

We can now state our main result, a bound on the competitive ratio of OBD for strongly convex costs. 
\begin{theorem}
OBD is competitive for the problem (\ref{oco-prob}) for all $\beta > \frac{4}{m}$. Furthermore, if $\beta$ is set to be  $2 + \frac{10}{m}$, the competitive ratio of OBD is at most $3 + O(1/m)$, irrespective of the ambient dimension. \label{cr-thm}
\end{theorem}



We note that \cite{pmlr-v75-chen18b} proved a bound on the competitive ratio of OBD of the form $3 + O(1/\alpha)$ where $\alpha$ measures the ``steepness'' of the costs. While this superficially resembles the bound in Theorem \ref{cr-thm}, we emphasize that the settings are quite different; their work applied to the class of polyhedral cost functions while we focus on strongly convex cost functions. In the case where the cost functions have minimum value zero these classes are wholly disjoint. We are led to consider strongly convex costs due to the fact many common learning and control problems have loss functions that are strongly convex (e.g., see Section 5).  Until this paper, there existed no competitive algorithms for SOCO problems with strongly convex costs.

To prove Theorem \ref{cr-thm}, we use the potential function $\phi(x, x^*) = \eta \|x - x^* \|^2$. Clearly $\phi(x, x^*) \geq 0$ and $\phi(x_0, x^*_0) = 0$. Before we turn to the proof of Theorem \ref{cr-thm}, we prove a series of crucial lemmas relating to the potential function. Lemmas \ref{potentialwiththeta-lemma} and \ref{potentialwithouttheta-lemma} show how the potential changes depending on the relative positions of its arguments, and highlight the role of the geometry associated with the $\ell_2$ norm. Lemma 3 relates the potential to the hitting costs at every timestep. 

\begin{lemma}

The change in potential satisfies $$\phi(a, c) - \phi(a, b) \leq -\phi(b, c)$$ for all $a, b, c \in \mathbb{R}^d$ such that the angle $\theta$ between the vectors $a - c$ and $b - c$ lies in  $[\pi/2, 3\pi/2]$. \label{potentialwiththeta-lemma}
\end{lemma}
\begin{proof}
Consider the triangle with vertices $a, b, c$. According to the Law of Cosines we have: $$\|a - b \|^2 = \|a - c\|^2 + \|b - c\|^2 - 2\|a - c\|\|b - c\|\cos{\theta}.$$ Rearranging gives 
$$\|a - c\|^2 - \|a - b \|^2 =  - \|b - c\|^2 + 2\|a - c\|\|b - c\|\cos{\theta}.$$ Since $\theta$ lies in  $[\pi/2, 3\pi/2]$, the cosine term must be non-positive, immediately yielding the claim.
\end{proof}

\begin{lemma}

The change in potential satisfies $$\phi(a, c) - \phi(a, b) \leq 2\phi(b, c) + \phi(a, b)$$ for all $a, b, c \in \mathbb{R}^d$.
\label{potentialwithouttheta-lemma}
\end{lemma}

\begin{proof}
We apply the Law of Cosines again: $$\|a - c \|^2 - \|a - b\|^2 = \|b - c\|^2 - 2\|a - b\|\|b - c\|\cos{\theta}$$ where $\theta$ is the angle between the vectors $a - b$ and $b - c$. The second term on the right is at most $2\|a - b\|\|b - c\|$; applying the AM-GM inequality to this expression gives the claim.
\end{proof}
\begin{lemma}
At all timesteps $t$, the potential satisfies $$\phi(x_t, x_t^*) \leq \frac{4\eta}{m}H_t + \frac{4\eta}{m}H_t^* $$ \label{potential-lem}
\end{lemma}
\begin{proof}
We have \begin{eqnarray*}
\phi(x_t, x_t^*) &=& \eta \|x_t - x_t^* \|^2\\
&\leq& \eta (\|x_t - v_t \| + \|x_t^* - v_t \|)^2\\
&\leq& 2\eta \|x_t - v_t\|^2 + 2\eta\|x_t^* - v_t\|^2 \\
&\leq& \frac{4\eta}{m}H_t + \frac{4\eta}{m}H_t^* \\
\end{eqnarray*}
The first inequality is just the triangle inequality; the second follows from the AM-GM inequality; and in the last step we used the fact that $f_t(x) \geq \frac{m}{2}\|x - v_t\|^2$.
\end{proof}

Now we return to the proof of Theorem \ref{cr-thm}. Note that it suffices to show that OBD is constant competitive in the case where minimum value of each cost function is zero, since otherwise the competitive ratio can only improve. In this case, we always have $M_t = \beta H_t$ since $H_t$ shrinks to zero as we move towards the minimizer while $M_t$ increases.

\begin{proof}
To bound the cost charged to OBD in each step, we first consider two cases.

\subsection*{Case 1: $H_t \leq H_t^*$}

This case is easy; the cost charged to OBD is 
\begin{eqnarray*}
H_t + M_t + \Delta \phi &\leq& H_t + M_t + \frac{4\eta}{m}H_t + \frac{4\eta}{m}H_t^* \\
&\leq& \left(1 + \beta + \frac{8\eta}{m}\right) H_t^*
\end{eqnarray*}
Here in the first inequality we threw away the negative potential term and used Lemma \ref{potential-lem}. In the second inequality we used the fact that $M_t  = \beta H_t$ and the inequality defining the case.

\subsection*{Case 2: $H_t > H_t^*$}
This is the hard case. Unlike in the previous case, we cannot directly bound the cost charged to OBD in terms of the offline cost, since $H_t^*$ is less than $H_t$. Our strategy will be to show that the change in potential was negative, offsetting the hitting and movement costs incurred by OBD.

Since $H_t^* < H_t$, the offline point $x_t^*$ must lie strictly in the interior of $K_t$, where $K_t$ is the $H_t$-level set of $f_t$. Notice that the angle $\theta$ made between the line segments $\overline{x_{t-1}x_{t}}$ and $\overline{x_{t} x_{t}^*}$ must be obtuse, since $x_t$ was the projection onto the level set, and $x_t^*$ lies strictly on the opposite side of the supporting hyperplane tangent to $K_t$ at $x_t$ (see Figure \ref{triangle-fig}). We have 
\begin{figure}
\includegraphics[width=\linewidth]{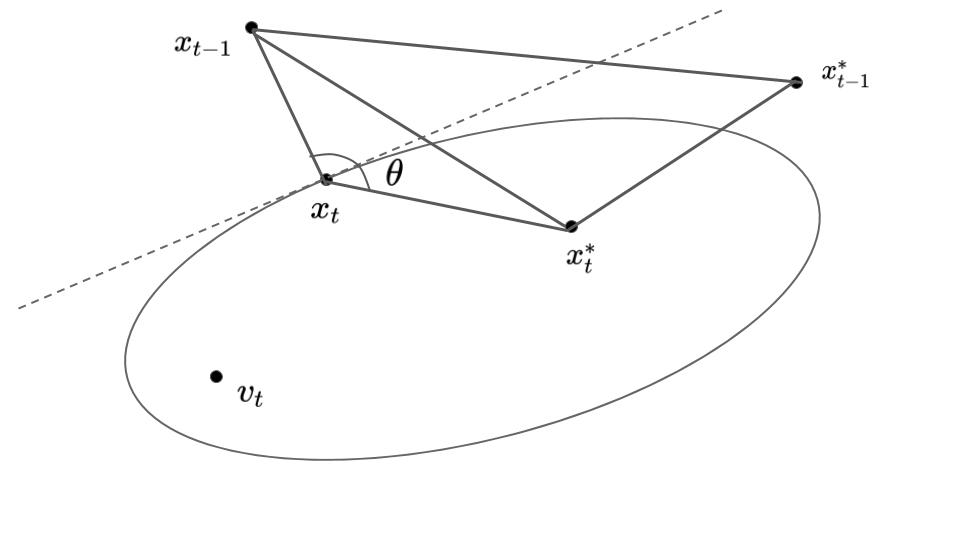}
\caption{An illustration of the situation arising in Case 2. The ellipse represents the level set $K_t$ of $f_t$ which OBD has selected; $x_t$ is the projection of $x_{t-1}$ onto this level set. The offline point $x_t^*$ must lie strictly inside the level set, since $H_t^* < H_t$. The angle $\theta$ must be obtuse, since $x_t^*$ lies on the other side of the supporting hyperplane at $x_t$, shown by the dotted line.} 
\label{triangle-fig}
\end{figure}
\begin{eqnarray*}
H_t + M_t + \Delta \phi &=& H_t + M_t + \left(\phi(x_t, x_t^*) - \phi(x_{t-1}, x_t^*) \right) \\
&& + \left(\phi(x_{t-1}, x_t^*) - \phi(x_{t-1}, x_{t-1}^*) \right) \\
&\leq& H_t + M_t - \phi(x_t, x_{t-1}) \\
&& + 2\phi(x_{t}^*, x_{t - 1}^*) + \phi(x_{t-1}, x_{t-1}^*)  \\
&=& H_t + M_t - \eta M_t + 2\eta M_t^*  \\
&& + \phi(x_{t-1}, x_{t-1}^*) \\
&\leq& 
\left(1  + \frac{1}{\beta} - \eta \right)M_t  + 2\eta M_t^* \\
&& + \frac{4\eta}{m} H_{t- 1} + \frac{4\eta}{m} H_{t- 1}^*
\end{eqnarray*}
In the first inequality we use Lemma \ref{potentialwiththeta-lemma} to bound the first change in potential and Lemma \ref{potentialwithouttheta-lemma} to bound the second.  In the second inequality we apply the fact that $M_t  = \beta H_t$ and  Lemma \ref{potential-lem}.

\subsection*{Bounding the competitive ratio}
We have now bounded the cost charged to OBD in each of the two cases. Putting both cases together, we see that we always have  
\begin{eqnarray*}
H_t + M_t + \Delta \phi &\leq&  \left(1 + \beta + \frac{8\eta}{m}\right)H_t^* + 2\eta M_t^*  \\
&& + \frac{4\eta}{m}H_{t-1} + \frac{4\eta}{m}H_{t-1}^*\\
\end{eqnarray*}
where we assume that $\beta, \eta$ were picked so that $\eta \geq 1 + \frac{1}{\beta}$. Adding up across all timesteps and collecting terms, we have:
\begin{eqnarray*}
\sum_{t = 1}^T H_t + M_t
&\leq&  \sum_{t = 1}^T \left(1 + \beta + \frac{12\eta}{m}\right)H_t^*  \\
&&+ \sum_{t = 1}^T 2\eta M_t^* + \sum_{t = 1}^T \frac{4\eta}{m}H_{t-1}\\
\end{eqnarray*} 
By the balance condition, $H_t = \frac{1}{1 + \beta}\left(H_t + M_t \right)$. We immediately obtain
\begin{eqnarray*}
ALG &\leq&   \max{ \left(\left(1 + \beta + \frac{12\eta}{m}\right), 2\eta \right) } OPT \\
&& +  \frac{4\eta}{m(1 + \beta)}ALG\\
\end{eqnarray*} 
Let us assume that  that $\beta, \eta$ are picked so that $ \frac{4\eta}{m(1 + \beta)} < 1$; rearranging gives 
\begin{eqnarray*}\frac{ALG}{OPT}
&\leq&  \frac{\max{ \left( \left(1 + \beta + \frac{12\eta}{m}\right), 2\eta \right)}}{1 - \frac{4\eta}{m(1 + \beta)}}   \\
\end{eqnarray*} 
Now we seek to minimize the competitive ratio by appropriately picking $\beta, \eta$, subject to the constraints. Notice that the competitive ratio is always increasing in $\eta$. Since we know that we must have $\eta \geq 1 + \frac{1}{\beta}$, this must be the optimal value of $\eta$. We can hence immediately rewrite our optimization problem purely in terms of $\beta$:
\begin{eqnarray*} \min_{\beta > \frac{4}{m}}  \frac{\max{ \left( \left(1 + \beta + \frac{12}{m}(1 + \frac{1}{\beta}\right), 2(1 + \frac{1}{\beta}) \right)}}{1 - \frac{4}{m\beta}} \\
\end{eqnarray*}
Note that this proves that OBD is competitive for all $\beta > \frac{4}{m}$. Instead of trying to find the exact optimal solution, we instead select a simple choice of $\beta$ which gives a small competitive ratio. Setting $\beta = 2 + \frac{10}{m}$ immediately gives an upper bound on the competitive ratio of $3 + O(1/{m})$ as claimed. 
\end{proof}

\section{Beyond-Worst-Case Analysis} \label{bwc-sec}
In the previous section we showed a worst-case performance bound on OBD and proved that the aggregate cost incurred by OBD is not more than a constant times the optimal aggregate cost. However, in most real world scenarios, the cost functions are not adversarial, prompting us to study the performance of OBD from a beyond-worst-case perspective. 

The difficulty in SOCO arises from the fact that the learner incurs switching costs in the face of cost functions which could change arbitrarily between rounds; yet in many practical settings the costs change slowly. This motivates the following definition:

\begin{definition} A sequence of points $s_1 \ldots s_T \in \mathbb{R}^d$ is $\epsilon$-smooth if $\|s_{t+1} - s_{t}\| \leq \epsilon$ for all $t = 1 \ldots T - 1.$ A sequence of convex functions $f_1 \ldots f_T$ is $\epsilon$-smooth if the corresponding sequence of minimizers $v_1 \ldots v_T$ is $\epsilon$-smooth.\label{smooth-def}
\end{definition}

Smooth instances have received considerable attention in the study of OCO, e.g. \cite{li2018using}, \cite{zinkevich2003online}.  Here, we show two interesting properties of OBD when the costs are $\epsilon$-smooth. 

First, we prove a  \textit{per-round accuracy} guarantee, showing that the point $x_t$ picked by OBD is always close to the minimizer $v_t$:


\begin{theorem} Suppose $v_1 \ldots v_T$ is $\epsilon$-smooth. Then the sequence of points $x_1 \ldots x_T$ picked by OBD is smooth with parameter $\left(1 + 2\alpha \right) \epsilon$ where $\beta > \frac{4}{m}$ is the balance parameter of OBD and $\alpha = \frac{1}{\sqrt{\beta m } - 2}$. Furthermore, the points $x_t$ chosen by OBD are always close to the current minimizer: $\| x_t - v_t \| \leq  \alpha \epsilon$ for all $t$. \label{smooth-thm}
\end{theorem}

This lets us bound the accuracy loss due to managing switching costs, guaranteeing that in each round we are not too far from the minimizer, despite the coupling across rounds. We note that when we set $\beta = 2 + \frac{10}{m}$ we get $\alpha < \frac{1}{\sqrt{10} - 2} \approx 0.86$. This gives an explicit numerical bound for the per-step accuracy when the balance parameter is set as in Theorem \ref{cr-thm}.

Secondly, we show that OBD incurs low dynamic regret when the costs are smooth: 

\begin{theorem}
Suppose $v_1 \ldots v_T$ is $\epsilon$-smooth, and fix balance parameter $\beta > \frac{4}{m}$. The dynamic regret of OBD is $O((\epsilon + \epsilon^2) T).$ \label{regret-thm}
\end{theorem}

We note that this result is nearly tight: Theorem 3 of \cite{li2018using} implies that no online algorithm can have dynamic regret better than $\Omega(\epsilon T)$. We note that \cite{pmlr-v75-chen18b} also proved a bound on the dynamic regret of OBD in terms of the smoothness of the cost sequence, but that bound grew super-linearly in $T$ in the worst case, even when $\epsilon$ is fixed. It is interesting to note that Theorem \ref{regret-thm} holds only when $\beta > \frac{4}{m}$, the same condition under which OBD is competitive. 

We end this section with proofs of Theorems \ref{smooth-thm} and \ref{regret-thm}.
\begin{proof}
By the balance condition and strong convexity, we have $$\frac{\beta m}{2}\|x_t - v_t \|_2^2 \leq \frac{1}{2} \|x_t - x_{t-1} \|_2^2 $$
Taking the square root of both sides and applying the triangle inequality gives 
$$\sqrt{\beta m}\|x_t - v_t \| \leq   \|x_t - v_t \| + \|v_t - v_{t-1} \| + \|v_{t-1} - x_{t-1} \| \ $$ 
from which
$$\|x_t - v_t \| \leq \frac{\epsilon + \|x_{t-1} - v_{t-1}\|}{\sqrt{\beta m } - 1}$$
 Unraveling this recursion gives $$\|x_t - v_t \| \leq  \left( \sum_{i = 1}^{t}  \frac{1}{(\sqrt{\beta m } - 1)^{t-i+1}} \right) \epsilon \leq \left(\frac{1}{\sqrt{\beta m } - 2} \right) \epsilon $$
where we summed the geometric series in the last step. Using the triangle inequality we immediately obtain
$$\|x_t - x_{t-1} \|  \leq \left(1 + 2\alpha \right) \epsilon $$ 
\end{proof}


Before we prove Theorem \ref{regret-thm}, we characterize how far the offline player moves relative to how smooth the cost sequence is:

\begin{lemma} \label{offline-lemma}
The offline points $x_1^* \ldots x_T^*$ satisfy $$\sum_{t = 1}^T \|x_t^* - v_t \| \leq \frac{2}{m}\sum_{t = 1}^T \|v_t - v_{t-1}\| $$ 
\end{lemma}
\begin{proof}
The first order condition is 
\begin{eqnarray*}
\nabla f_t(x_t^*) + x_t^* - x_{t-1}^* + x_t^* - x_{t + 1}^* &=& 0
\end{eqnarray*}
for $t = 1 \ldots T - 1$, and for the last timestep is
\begin{eqnarray*}
\nabla f_T(x_T^*) + x_T^* - x_{T-1}^* &=& 0
\end{eqnarray*}
We add and subtract $2v_t - (v_{t-1} + v_{t+1})$ from the first set of equations and $v_T - v_{T-1} $ from the last equation, and right-multiply the resulting equations by the vectors $(x_t^* - v_t)$ and $(x_T^* - v_T)$, respectively. Applying the Cauchy-Schwartz Inequality and strong convexity  
\begin{eqnarray*}
\nabla f_t(x)^{\top}(x - v_t) \geq m\|x - v_t\|_2^2 
\end{eqnarray*}
we eventually obtain  
\begin{eqnarray*}
\delta_t(\delta_{t-1} + \delta_{t + 1} + \epsilon_t  + \epsilon_{t+1}) &\geq (m + 2)\delta_t^2
\end{eqnarray*}
for $t = 1 \ldots T - 1$, and in the last timestep
\begin{eqnarray*}
\delta_T(\delta_{T-1} + \epsilon_T ) &\geq (m + 1)\delta_T^2,
\end{eqnarray*}
where $\delta_t = \|x_t^* - v_t\|, \epsilon_t = \|v_t - v_{t-1}\|$. Dividing both sides by $\delta_t$ and summing up over $t$ leads to the claim. 
\end{proof}

Now let us return to the proof of Theorem \ref{regret-thm}. As is standard when proving regret bounds, we assume that the gradients of $f_t$ over the action set are uniformly bounded by a constant $G$, e.g. see \cite{zinkevich2003online}.

\begin{proof}
The regret is 
\begin{eqnarray*}
ALG - OPT &=& \sum_{t = 1}^T \left[ f_t(x_t) - f_t(x^*) + \frac{1}{2}\eta_t^2 - \frac{1}{2}(\eta_t^*)^2 \right] \\
&\leq&  \sum_{t = 1}^T \left[ G\|x_t - x^*_t\| + \frac{\eta_t^2}{2} \right] 
\end{eqnarray*}
where in the inequality we applied the Lipschitz property and threw away the offline's movement cost. Noticing that $\|x_t - x_t^*\| \leq \|x_t - v_t\| + \|x_t^* - v_t\| $ allows to apply Theorem \ref{smooth-thm} and Lemma \ref{offline-lemma} to obtain the bound
\begin{eqnarray*}
\left[ G\alpha \epsilon  + G\frac{2}{m}\epsilon + \frac{1}{2} \left(1 + 2\alpha \right)^2 \epsilon^2  \right]T \\
\end{eqnarray*}
where $\alpha = \frac{1}{\sqrt{\beta m } - 2}$. This is $O((\epsilon + \epsilon^2) T)$ as claimed.
\end{proof}

\section{Applications} \label{applications-sec}
In this section we show several applications of OBD to diverse problems across learning and control. We emphasize that none of these applications would be possible without a competitive algorithm for strongly convex costs, which has not been attainable with previous approaches. 

\subsection{Smoothed Online Regression }
We consider the problem of a learner who wishes to fit a series of regularized regressors or classifiers to a changing dataset, without changing the estimators too much between rounds. This is naturally modeled by the objective
\begin{eqnarray}
\min_{\theta_1 \ldots \theta_T \in \mathbb{R}^d} \sum_{t = 1}^T  f_t(\theta_t) + \frac{\lambda_1}{2} \| \theta_t  \|_2^2 + \frac{\lambda_2}{2} \|\theta_t - \theta_{t-1} \|_2^2 \label{stat-prob}
\end{eqnarray}
Here $f_t$ represents an estimation or classification task at timestep $t$, $\theta_t$ is the regressor at time $t$, and $\lambda_1, \lambda_2$ are parameters that control the strength of the $\ell_2$ and smoothing regularizations, respectively. We impose no constraint on $f_t$ other than convexity; in particular, $f_t$ need not be strongly convex (though if $f_t$ happens to be strongly convex, we can optionally drop the regularization term $\frac{\lambda_1}{2} \| \theta_t \|^2$). OBD gives a constant-competitive algorithm in this setting:

\begin{corollary}
The competitive ratio of OBD with balance parameter $\beta = 2 + \frac{10}{m}$ on problem \ref{stat-prob} is $3 +  O(\lambda_2 / \lambda_1)$. 
\end{corollary}
Before we turn to the proof, we emphasize that the bound on competitive ratio does not vary with respect to dimension; hence OBD can be applied to estimation problems with thousands or millions of parameters.
\begin{proof}
We first divide the objective by $\lambda_2$. Notice that the function $\tilde{f_t}(\theta) = \frac{1}{\lambda_2}f_t(\theta) + \frac{\lambda_1}{2\lambda_2} \| \theta_t  \|_2^2 $ is $(\lambda_1/\lambda_2)$-strongly convex in $\theta$ whenever $f_t$ is convex,  hence Theorem \ref{cr-thm} implies that OBD achieves competitive ratio $3 + O(\lambda_2 / \lambda_1)$.
\end{proof}

Our approach applies to many common learning problems:

\begin{itemize}
\item \textbf{Ridge Regression}. We take $f_t(\theta) =  \|X_t \theta - y_t \|^2$, where $X_t \in \mathbb{R}^{n \times d}$ is a data matrix and $y_t \in \mathbb{R}^n$ is the response variable. 

\item \textbf{Logistic Regression}. We take $f_t(\theta) = -\frac{1}{n_t}\sum_{i = 1}^{n_t} \log{(1 + e^{-y_{i, t}\theta^{\top} x_{i, t}})} $ where $x_{i, t} \in \mathbb{R}^d$ is a vector of features, $y_{i, t} \in \{0, 1\}$ is a binary outcome, and $n_t$ is the number of samples in round $t$. OBD hence fits a series of binary classifiers which don't vary too much between rounds. Our approach easily extends to the multiclass setting as well.

\item \textbf{Maximum Likelihood Estimation.}
More generally, we can perform smoothed online maximum likelihood estimation using OBD.
Here $\theta_t$ are parameters of a model and $f_t(\theta_t)$ is the likelihood function of some dataset at time $t$. If the likelihood function is convex then OBD can be applied. For example, the problem of estimating a series of covariance matrices $\Sigma_t$ of a series of Gaussian distributions $\mathcal{N}(0, \Sigma_t)$ given independent samples arranged as a matrix $Y_t$ can be posed as the problem $$\min_{\Sigma} \Tr(\Sigma Y) - \log{\det{\Sigma}}$$ which is a convex optimization problem (see \cite{boyd2004convex}, p. 357). We can apply OBD over the set of positive definite matrices to find a series of covariance matrices $\Sigma_t$ that fit the data well but don't vary too much between rounds.

\end{itemize}

\subsection{Linear Quadratic Regulator (LQR) control}
Our second application comes from the controls community. Consider the classical problem of LQR control: $$\min_{u_1 \ldots u_T}\sum_{t = 1}^T \frac{1}{2} x_t'Q_tx_t + \frac{1}{2} u_t'R u_t $$ with dynamics given by $$ x_{t + 1} = Ax_t + Bu_t + w_t $$
Here $u$ is a control action, $x$ a state variable, and $Q_t, R $ are assumed to be positive definite. Usually, the noise increments $w_t$ are assumed to be i.i.d. Gaussian, and the goal is to design a control policy to minimize the expected cost. Instead of an in-expectation result, we can use OBD to design a controller with a strong pathwise guarantee, with no distributional or boundedness assumptions on the noise. We focus on the setting where $A = I$, i.e. the system is stationary in the absence of noise or control actions.

\begin{corollary}
Suppose that $A = I$ and $B$ is invertible, and the matrices $Q_t$ each have their lowest eigenvalue bounded below by $\lambda > 0$. The LQR problem can be rewritten as a SOCO problem, and the competitive ratio of OBD is $$3 + O\left(\frac{\lambda_{max}(R)}{\lambda_{min}(B)^2\lambda}\right).$$
\end{corollary}

Note that $\lambda_{min}(B)$ can be interpreted as a lower bound on the gain of the control action $u$; intuitively, systems with high control gain are easier to regulate, since each control action gets amplified. Similarly, it is intuitive that as $\lambda_{max}(R)$ decreases the competitive ratio improves, since $R$ controls the cost incurred by using the controller.

\begin{proof}
Define $$ y_t = \sum_{i = 1}^t Bu_i, \hspace{10mm} v_t = -\sum_{i = 1}^t w_i.$$ Notice that  $$x_t = y_t - v_t, \hspace{10mm} u_t = B^{-1} (y_t - y_{t-1}),$$
so the LQR problem can be rewritten as 
$$\min_{y_1 \ldots y_T}\sum_{t = 1}^T \frac{1}{2} (y_t - v_t)'Q_t(y_t - v_t) + \frac{1}{2} (y_t - y_{t-1})'Z(y_t - y_{t-1})$$ where $$Z = (B^{-1})' R B^{-1}$$

Now define $$z_t = R^{\frac{1}{2}} B^{-1}y_t, \hspace{10mm} s_t = R^{\frac{1}{2}} B^{-1} v_t. $$  The optimization problem becomes $$\min_{z_1 \ldots z_T} \sum_{t = 1}^T \frac{1}{2} (z_t  - s_t)' P_t (z_t - s_t) + \frac{1}{2} \|z_t - z_{t-1}\|_2^2 $$ where $$P_t = (BR^{-\frac{1}{2}})' Q_t BR^{-\frac{1}{2}} $$ 
This is just a special case of the SOCO problem. Notice that the costs are strongly convex with parameter $\lambda_{min}(P_t)$, which is bounded below by $$\frac{\lambda_{min}B^2\lambda}{\lambda_{max}(R)}$$ which in light of Theorem \ref{cr-thm} proves the claim.

\end{proof}

\section{Concluding Remarks}
We show in this paper that the OBD algorithm is constant-competitive for SOCO with strongly convex costs, making OBD the first competitive algorithm in this setting. We also show that when the sequence of cost functions is smooth, OBD maintains good per-round accuracy and near-optimal dynamic regret. Finally, we apply OBD to a variety of important learning and control problems, including online maximum likelihood estimation and LQR control, giving the first constant competitive algorithms for these problems. 

We conclude by identifying two important open problems in the area of smoothed online learning. First, it would be valuable if OBD were able to handle different kind of smoothing regularizers. For example, the $\ell_1$ smoothing regularizer $\|x_t - x_{t-1}\|_1$ could be used to promote sparsity in the updates between rounds. Secondly, it is natural to extend the OBD framework to non-convex problems such as matrix completion and tensor factorization, which are increasingly popular in machine learning.    

\newpage





\bibliographystyle{alpha}
\bibliography{AISTATS2019}

\end{document}